\documentclass[11pt,a4paper]{article}

\usepackage[top=2.5cm,bottom=2.5cm,left=2.8cm,right=2.8cm]{geometry}
\usepackage[protrusion=false,expansion=false]{microtype}
\usepackage{setspace}
\usepackage[T1]{fontenc}
\usepackage[utf8]{inputenc}

\usepackage{amsmath,amssymb,amsthm}

\usepackage{booktabs}
\usepackage{tabularx}
\usepackage{array}
\usepackage{multirow}

\usepackage{listings}
\usepackage[svgnames]{xcolor}

\lstdefinestyle{prologstyle}{
  language=Prolog,
  basicstyle=\ttfamily\small,
  keywordstyle=\color{Navy}\bfseries,
  commentstyle=\color{ForestGreen},
  stringstyle=\color{Maroon},
  backgroundcolor=\color{AliceBlue},
  frame=single,
  framerule=0.4pt,
  rulecolor=\color{SteelBlue!50},
  framesep=5pt,
  xleftmargin=8pt,
  xrightmargin=4pt,
  breaklines=true,
  showstringspaces=false,
  tabsize=2,
  captionpos=b,
  numbers=none,
}

\lstdefinestyle{sparqlstyle}{
  basicstyle=\ttfamily\small,
  keywordstyle=\color{Navy}\bfseries,
  commentstyle=\color{ForestGreen},
  backgroundcolor=\color{AliceBlue},
  frame=single,
  framerule=0.4pt,
  rulecolor=\color{SteelBlue!50},
  framesep=5pt,
  xleftmargin=8pt,
  xrightmargin=4pt,
  breaklines=true,
  showstringspaces=false,
  tabsize=2,
  numbers=none,
}

\lstdefinestyle{genericstyle}{
  basicstyle=\ttfamily\small,
  backgroundcolor=\color{AliceBlue},
  frame=single,
  framerule=0.4pt,
  rulecolor=\color{SteelBlue!50},
  framesep=5pt,
  xleftmargin=8pt,
  xrightmargin=4pt,
  breaklines=true,
  showstringspaces=false,
}

\theoremstyle{definition}
\newtheorem{definition}{Definition}[section]

\theoremstyle{plain}
\newtheorem{theorem}{Theorem}[section]
\newtheorem{property}{Property}[section]

\theoremstyle{remark}

\usepackage{graphicx}
\usepackage{float}

\usepackage[colorlinks=true,
            linkcolor=NavyBlue,
            citecolor=DarkGreen,
            urlcolor=DarkBlue]{hyperref}

\usepackage{enumitem}
\setlist[itemize]{itemsep=2pt,topsep=4pt}
\setlist[enumerate]{itemsep=2pt,topsep=4pt}
\usepackage{xspace}
\usepackage{parskip}
\usepackage{titlesec}
\titleformat{\section}{\large\bfseries}{\thesection.}{0.6em}{}
\titleformat{\subsection}{\normalsize\bfseries}{\thesubsection}{0.6em}{}
\titleformat{\subsubsection}{\normalsize\itshape}{\thesubsubsection}{0.6em}{}
\titlespacing{\section}{0pt}{12pt}{6pt}
\titlespacing{\subsection}{0pt}{8pt}{4pt}

\usepackage{tcolorbox}
\tcbuselibrary{skins}
\newcommand{\chainnode}[2]{%
  \begin{tcolorbox}[
    enhanced, boxrule=0.4pt, arc=3pt,
    colback=SteelBlue!8, colframe=SteelBlue!50,
    left=4pt, right=4pt, top=2pt, bottom=2pt,
    width=\linewidth, nobeforeafter
  ]
  \centering\small
  \textbf{\texttt{#1}}\\[1pt]
  \textit{\footnotesize #2}
  \end{tcolorbox}%
}
\newcommand{\chainarrow}[1]{%
  \vspace{1pt}%
  \begin{center}
    {\color{SteelBlue!80}$\downarrow$}\enspace
    {\footnotesize\itshape\color{gray!80!black}#1}
  \end{center}%
  \vspace{1pt}%
}

\usepackage{caption}
\begin{document}

\begin{center}
  {\LARGE\bfseries Domain-Constrained Knowledge Representation:\\ A Modal Framework}

  \vspace{1.2em}

  {\normalsize
    Chao Li\textsuperscript{1}\quad
    Yuru Wang\textsuperscript{2}\quad
    Chunyi Zhao\textsuperscript{3}
  }

  \vspace{0.5em}
  {\small
    \textsuperscript{1}Deepleap.ai \quad \texttt{lichao@deepleap.ai}\\[2pt]
    \textsuperscript{2}School of Information Science and Technology,
      Northeast Normal University \quad \texttt{wangyr915@nenu.edu.cn}\\[2pt]
    \textsuperscript{3}Centre of Educational Design and Learning,
      University of Otago \quad \texttt{cccchunyi07@gmail.com}
  }
\end{center}

\vspace{1em}
\hrule
\vspace{0.8em}

\begin{abstract}
Knowledge graphs store large numbers of relations efficiently, but they remain
weak at representing a quieter difficulty: the meaning of a concept often shifts
with the domain in which it is used. A triple such as $\langle\text{Apple},
\text{instance-of}, \text{Company}\rangle$ may be acceptable in one setting
while being misleading or unusable in another. In most current systems, domain
information is attached as metadata, qualifiers, or graph-level organization.
These mechanisms help with filtering and provenance, but they usually do not
alter the formal status of the assertion itself.

This paper argues that domain should be treated as part of knowledge
representation rather than as supplementary annotation. It introduces the
\emph{Domain-Contextualized Concept Graph} (CDC), a framework in which domain
is written into the relation and interpreted as a modal world constraint. In the
CDC form $\langle C,\, R@D,\, C'\rangle$, the marker $@D$ identifies the world
in which the relation is licensed to hold. Formally, the relation is interpreted
through a domain-indexed necessity operator $\Box_D$, so that truth, inference,
and conflict checking are all scoped to the relevant world.

This move has three immediate consequences. First, ambiguous concepts can be
disambiguated at the point of representation rather than deferred to later
processing. Second, invalid or ill-formed assertions can be challenged against
the domain they invoke. Third, structurally comparable relations across domains
can be connected through explicit cross-domain predicates rather than ad hoc
alignment alone. The paper develops this claim through a Kripke-style semantics,
a compact predicate system, a Prolog reference implementation, and mappings to
RDF stores, OWL ontologies, and relational databases.

The contribution of the paper is therefore not another context tag, but a
representational reinterpretation of domain itself. The central claim is that
many practical failures in knowledge systems begin when domain is treated as
external to the assertion. CDC addresses that problem by giving domain a
structural and computable role inside the representation.

\end{abstract}

\vspace{0.8em}
\hrule
\vspace{1em}

\tableofcontents

\newpage

\section{Introduction}

\subsection{The Problem: Three Fields, One Unresolved Assumption}

Knowledge graphs have transformed the storage and retrieval of structured
information~\cite{hogan2021}. Systems such as Google's Knowledge Graph,
DBpedia~\cite{auer2007}, and Wikidata~\cite{vrandecic2014} show the practical
power of triples and graph-based schemas at scale. Yet the efficiency of the
triple model conceals a persistent weakness: a triple says nothing about the
world in which it is supposed to be interpreted.

When DBpedia records $\langle\text{Apple}, \text{instance-of},
\text{Company}\rangle$, a human reader may infer the intended domain without
difficulty. The formal structure cannot enforce that reading, cannot determine
when Apple should be treated as a fruit, and cannot prevent one reading from
contaminating another. This is not missing information; it is missing
structure. The representation does not encode the conditions under which meaning
holds.

The same difficulty appears wherever concepts cross domains. A researcher
comparing artificial neural networks with biological neural systems encounters a
representational gap: in one ontology a neural network is an algorithmic model,
while in another the brain is a biological organ. The structural relation may be
intellectually productive, but the knowledge graph does not know how to preserve
it without manual alignment~\cite{euzenat2013}. Historical concepts create the
same pressure. Dalton, Thomson, Bohr, and contemporary physics do not simply add
detail to one unchanged object; they reframe what ``atom'' means. Standard
ontologies address this by replacing definitions or splitting the concept into
separate entities such as \texttt{Atom\_Dalton} and \texttt{Atom\_Bohr}~\cite{flouris2008}. What is
lost is the ability to represent continuity and difference together.

The problem extends beyond knowledge graphs. OWL and description logics adopt an
open-world assumption with unary semantics: a class assertion holds globally,
with no native construct for expressing that an assertion is valid only within a
particular context~\cite{baader2007}. This is not a deficiency of OWL; it is its
design premise: serve ontological consistency and decidable reasoning, not
context-sensitive knowledge. But in practice, it becomes a hard constraint
against alignment with complex domains such as medicine, engineering, and law.
The limitation has motivated two decades of extensions---contextualized
ontologies~\cite{bozzato2018}, modular ontologies~\cite{grau2008}, and named
graphs---none of which makes context a structural parameter of the inference unit
itself.

ICD-11, the WHO's global standard for health information ($\sim$85\,000
entities), manifests the same problem at industrial scale. A single disease must
be classified simultaneously along anatomical, etiological, pathological, and
functional axes. The WHO's own technical documentation acknowledges that the
ontology layer ``was not completed and remains an opportunity for future
development''~\cite{chute2022}. The root cause: OWL's \texttt{subClassOf}
relation does not distinguish \emph{why} an entity belongs to a parent class.
Streptococcal pneumonia is simultaneously a respiratory disease (anatomical
axis), a bacterial infection (etiological axis), and a streptococcal disease
(pathogen axis). But the formalism treats these as three undifferentiated
parent-class memberships. Multiple inheritance becomes ambiguous polyhierarchy,
with no mechanism to scope classification to a specific axis.

\medskip\noindent These are not three separate problems. They are the same
problem in three settings.

\subsection{The Shared Root Cause}

Knowledge graphs, OWL ontologies, and medical classification systems share an
assumption so foundational that it has rarely been explicitly stated:
\emph{inference relations are global}. The relation \texttt{is\_a} has one
meaning everywhere. The domain in which an assertion holds---the semantic world
that licenses the relation---lives outside the representation: in
query-language \texttt{WHERE} clauses, in application code, and in the human
user's head. The inference engine does not know which world it is operating in.

A common response is to attach more context: qualifiers, metadata, named graphs,
provenance records. These devices are useful---they improve documentation,
support filtering, and help organize data. But in most systems the core assertion
remains formally untouched. The triple still stands as if it were globally valid,
while the domain appears alongside it as supplementary explanation.

The distinction this paper develops is therefore the distinction between a
\emph{label} and a \emph{constraint}. A label indicates how an assertion is
intended to be read. A constraint determines whether the assertion is licensed to
stand at all. Domain, we argue, belongs to the second category. If domain is
built into the relation as a structural constraint---as part of the predicate
arity, not as external annotation---then it becomes part of what the
representation means and part of how subsequent inference must proceed. The
system cannot ignore it, just as a function call cannot ignore one of its
arguments.

\subsection{Theoretical Foundations}

CDC is grounded in two complementary traditions.

\medskip\noindent\textbf{Cognitive-linguistic foundation.} Humans do not store
concepts as universal definitions; instead, we understand them through
domain-dependent frames~\cite{fillmore1982}: ``bank'' in finance activates a
frame of money, accounts, and transactions, while ``bank'' in geography activates
a frame of rivers, erosion, and terrain. This is not lexical ambiguity; it is
how cognition works~\cite{barsalou1982,gardenfors2000}. CDC operationalizes this
through a three-level isomorphic mapping between cognitive frames, linguistic
context markers, and computational domain specifications~\cite{goldberg2006,lakoff1980}.

\medskip\noindent\textbf{Modal logic foundation.} Since Kripke~\cite{kripke1963},
the insight that propositions can hold relative to possible worlds has been
central to logic. The necessity operator $\Box$ asserts that a proposition holds
in all accessible worlds. McCarthy~\cite{mccarthy1993} proposed formalizing
contexts as first-class logical objects. Fagin et al.~\cite{fagin2004}
demonstrated that modal frameworks can capture epistemic states in multi-agent
systems. But within knowledge representation (KR), these modal approaches have
remained theoretical; no prior work has embedded modal world constraints directly
into the representational layer of a knowledge graph as computable structural
elements. CDC fills this gap.

\subsection{Contributions}

This paper makes four contributions, each demonstrated through concrete examples
before being formalized.

\begin{enumerate}[leftmargin=*]
  \item \textbf{Modal constraint interpretation.} We show that CDC's domain
  annotation $@D$ functions as a modal necessity operator $\Box_D$---not merely
  analogously, but with formal consequences for what can be asserted, inferred,
  and rejected within a knowledge representation.

  \item \textbf{Representation-level power.} We demonstrate that $@D$ enables
  capabilities at the representational layer that traditional KR cannot achieve:
  falsification of ill-formed assertions, a priori disambiguation, cross-domain
  structural transfer, and multi-path constraint locking.

  \item \textbf{Constraint transferability.} We formalize the property that any
  computational system consuming CDC representations automatically inherits the
  semantic constraints of $@D$, making the representation an active constraint
  interface rather than passive data.

  \item \textbf{Computational validation.} We provide a Prolog reference
  implementation, side-by-side query comparison with RDF and Wikidata
  demonstrating CDC's structural advantages, and case studies---including ICD-11
  respiratory diseases---validating the framework's computability and practical
  value.
\end{enumerate}

This paper focuses on the representational foundations of CDC: what $@D$ means,
why it constitutes a constraint rather than a label, and what representational
properties follow from this interpretation. Questions of computational
complexity, decidability guarantees, and optimized inference algorithms over CDC
structures---while important---constitute a distinct line of investigation that
builds upon the representational semantics established here.

\section{Related Work}

\subsection{Knowledge Graph Systems}

DBpedia~\cite{auer2007} and Freebase~\cite{bollacker2008} established the
paradigm of fixed ontological schemas. Wikidata~\cite{vrandecic2014} introduced
qualifiers---temporal or situational annotations attached to statements. But
qualifiers do not change the logical status of assertions; they are metadata
\emph{about} statements, not structural constraints \emph{within} them.
Schema.org~\cite{guha2016} provides web-optimized structured vocabulary but is
designed for search-engine interoperability, not general-purpose knowledge
representation.

\subsection{Ontology Engineering}

Upper ontologies such as SUMO~\cite{niles2001}, Cyc~\cite{lenat1995}, and
BFO~\cite{arp2015} seek universal categories applicable across all domains.
They provide formal rigor but sacrifice adaptability: when a concept must mean
different things in different disciplines, universal categories become a
straitjacket. Domain-specific ontologies, such as SNOMED CT and Gene
Ontology~\cite{smith2007}, deliver rich internal structure but resist
cross-disciplinary integration. CDC positions itself as a cross-domain
integration layer. It does not replace domain ontologies, but connects them
through explicit modal constraints.

\subsection{Cognitive Science Foundations}

G{\"a}rdenfors' conceptual spaces~\cite{gardenfors2000} model concepts as
regions in multidimensional quality spaces. Fillmore's frame
semantics~\cite{fillmore1982} proposes that meanings are interpreted relative to
conceptual frames. Barsalou~\cite{barsalou1982} and Medin \&
Shoben~\cite{medin1988} demonstrate that human cognition systematically adapts
categorization to situational context. CDC operationalizes these insights: each
domain specification defines a cognitive frame within which concepts acquire
specific meanings.

\subsection{Logic-Based Representation}

Description Logic (DL) provides the semantic foundation of OWL
ontologies~\cite{baader2003,horrocks2003} and emphasizes decidable reasoning
under well-defined constraints. CDC's $\Box_D$ operator is formally compatible
with DL's existing modal extensions. Epistemic DLs already incorporate modal
operators for knowledge and belief, such as the epistemic operator $K$ in
DL-Lite, and temporal DLs use modal operators for time-indexed assertions.
CDC's contribution is not a new logical formalism but a specific
operationalization that treats the domain specification as a Kripke-style world
index within an otherwise standard relational structure. A CDC assertion
$\mathtt{is\_a}(C, C', D)$ can be read as a DL axiom $C \sqsubseteq_D C'$, a
subsumption that holds within world $D$. This compatibility means CDC does not
require abandoning DL-based infrastructure; it extends that infrastructure with
explicit world indexing.

RDF named graphs~\cite{carroll2005} extend triples to quads for provenance
tracking and statement grouping. RDF-star, in the forthcoming RDF 1.2, enables
annotations on individual triples through quoted triple syntax. RDF-star can
represent the same surface structure as CDC---a triple annotated with a domain
specification. The distinction is semantic rather than syntactic: RDF-star
annotations carry no formal semantics that participate in inference. Transitive
closure, attribute inheritance, and prerequisite chains do not automatically
respect quoted-triple annotations unless each inference rule is manually extended
to check them. CDC's $@D$ carries modal necessity semantics, $\Box_D$, that
structurally constrain reasoning. The same notation, without the modal
interpretation, is a label; with it, it is a constraint.

\subsection{Modal Logic in Knowledge Representation}

Kripke's possible-world semantics~\cite{kripke1963} provides the framework for
necessity and possibility relative to accessible worlds.
McCarthy~\cite{mccarthy1993} proposed contexts as formal objects. Situation
calculus~\cite{mccarthy1969} models how facts change across situations. Fagin et
al.~\cite{fagin2004} demonstrated modal reasoning about knowledge in multi-agent
systems. These approaches share CDC's concern with context-dependent truth. But
they have remained theoretical tools used to \emph{analyze} knowledge
representations, not to \emph{constitute} them. Description Logics incorporated
limited modal features (temporal extensions), but no prior work has made modal
world constraints a first-class, computable component of the knowledge graph
representation itself. This is CDC's distinctive contribution: $@D$ is not a
theoretical analysis tool applied to representations---it \emph{is part of} the
representation.

Table~\ref{tab:related} summarises CDC's relationship to prior work.

\begin{table}[ht]
\centering
\caption{Positioning CDC relative to existing approaches.}
\label{tab:related}
\small
\begin{tabularx}{\linewidth}{lllX}
\toprule
\textbf{Approach} & \textbf{Layer} & \textbf{CDC's Relationship} \\
\midrule
DBpedia / Freebase   & Implementation    & Provides data substrate \\
Wikidata             & Implementation    & Extends with domain-scoped reasoning \\
Upper ontologies     & Formal semantics  & More flexible philosophy \\
Domain ontologies    & Specialized       & Cross-domain integration layer \\
Conceptual spaces    & Geometric theory  & Cognitive inspiration \\
Frame semantics      & Linguistic theory & Theoretical grounding \\
Description logic    & Logical formalism & Trades decidability for expressiveness \\
Named graphs         & Technical         & Reinterprets as semantic context \\
Modal logic          & Formal theory     & Operationalizes as computable constraint \\
\bottomrule
\end{tabularx}
\end{table}

\subsection{Context-Aware Knowledge Representation}

The absence of a formal solution has not prevented deployment. It has produced a
catalog of workarounds, each addressing a symptom while leaving the root cause
intact. ICD-11's Foundation Layer permits multiple parent classes without
distinguishing which classification axis each parent represents---users are
expected to filter manually. Postcoordination encodes multi-axis classification
as string concatenation of base codes and qualifier codes, producing
combinatorial complexity that users must navigate by memorizing coding rules.
Hospital EMR systems hard-code axis logic in application layers---``respiratory
department queries return only anatomical classifications''---producing solutions
that work but cannot be reused, adapted, or maintained when classification axes
change. SWRL rules patch OWL reasoning with external logic, coupling tightly to
the ontology version and becoming unmaintainable at scale. And in many cases,
the multi-axis semantics are simply written in natural language documentation:
human-readable, machine-opaque.

These workarounds share a pattern. Each moves the domain information further
from the inference engine: from the representation (where it could constrain
reasoning) to application code, to coding conventions, to documentation, to the
clinician's memory. The result is that knowledge systems operate globally while
humans perform domain scoping manually---precisely the separation this paper
addresses.

\section{The CDC Structure}

\subsection{Formal Definition}

\begin{definition}[CDC Triple]
A CDC triple is a four-tuple:
\[
  \tau = \langle c,\; r,\; c',\; d \rangle
\]
where $c \in \mathcal{C}$ (source concept), $r \in \mathcal{R}$ (relation
predicate), $c' \in \mathcal{C}$ (target concept), and $d \in \mathcal{D}$
(domain specification).

\smallskip\noindent\emph{Notation:}
$c \xrightarrow{r@d} c'$,\quad or equivalently $r(c,\, c',\, d)$.
\end{definition}

\begin{definition}[Domain Specification]
A domain specification $d \in \mathcal{D}$ is a structured string:
\[
  d \;:=\; \mathit{dimension} \;\mid\; \mathit{dimension}@d
\]
\emph{Examples:}
\begin{itemize}[nosep,leftmargin=1.5em]
  \item \texttt{'Physics'}
  \item \texttt{'Physics@Quantum\_Mechanics'}
  \item \texttt{'HighSchool@Math@Calculus'}
  \item \texttt{'Student\_Zhang@Grade10'}
\end{itemize}

\smallskip\noindent\emph{Critical property:} Domains are defined on demand, not
from a fixed taxonomy.
\end{definition}

\subsection{Domain Patterns}

Table~\ref{tab:patterns} illustrates common domain specification patterns across
use cases.

\begin{table}[ht]
\centering
\caption{Domain specification patterns by use case.}
\label{tab:patterns}
\small
\begin{tabular}{lll}
\toprule
\textbf{Use Case} & \textbf{Pattern} & \textbf{Example} \\
\midrule
Academic Research & \texttt{Discipline@Theory}          & \texttt{'Psychology@Behaviorism'} \\
Education         & \texttt{Grade@Subject@Topic}        & \texttt{'HighSchool@Chemistry@Organic'} \\
Enterprise        & \texttt{Department@Project}         & \texttt{'Engineering@ProductA@Testing'} \\
Historical        & \texttt{Era@Region@Movement}        & \texttt{'Renaissance@Italy@Humanism'} \\
Technical Docs    & \texttt{Stack@Version@Context}      & \texttt{'React@18.x@Mobile\_Apps'} \\
Personal Learning & \texttt{Individual@Background}      & \texttt{'Student\_Li@CS\_Major'} \\
\bottomrule
\end{tabular}
\end{table}

\noindent\textbf{Design principles:} \emph{Specificity} (fine-grained enough to
disambiguate, not so fine as to fragment); \emph{Consistency} (internally
coherent patterns); \emph{Compositionality} (hierarchical where natural);
\emph{Context-Sensitivity} (capture the most relevant distinguishing context).

\subsection{Representation in Practice}

CDC allows divergent categorisations to coexist without contradiction:

\begin{lstlisting}[style=prologstyle]
is_a(Apple, Fruit,                'Biology@Plant_Taxonomy').
is_a(Apple, Company,              'Business@Technology_Industry').
is_a(Neural_Network, Computational_Model, 'CS@ML').
analogous_to(Neural_Network, Biological_Brain, 'CS@ML', 'Biology@Neuroscience').
is_a(Neural_Network, Function_Approximator, 'Math@Optimization').
is_a(Neural_Network, Philosophy_Topic,      'Philosophy@Mind').
\end{lstlisting}

This is the surface structure. Section~\ref{sec:power} reveals what makes it
powerful.

\subsection{Migration Path: From Existing Representations to CDC}

CDC does not require rebuilding knowledge from scratch. Any existing knowledge
representation---whether a traditional knowledge graph (RDF triples), an OWL
ontology, a relational database, or a structured document---can be migrated to
CDC through a straightforward transformation:

\begin{itemize}
  \item Existing entities become concepts $\mathcal{C}$.
  \item Existing relations become relation predicates $r$.
  \item Existing namespaces, categories, or contextual tags become domain
        specifications $\mathcal{D}$.
\end{itemize}

\noindent The mapping is direct:
\begin{center}
\begin{tabular}{lp{0.62\linewidth}}
\toprule
Traditional triple: & $\langle\text{Apple},\;\text{instance-of},\;\text{Company}\rangle$ \\[4pt]
CDC triple:         & $\langle\text{Apple},\;\text{is\_a},\;\text{Company},$\newline
                    \quad\texttt{'Business@Technology\_Industry'}$\rangle$ \\
\bottomrule
\end{tabular}
\end{center}

Hierarchical domains (e.g., \texttt{'Physics@Quantum\_Mechanics'}) can be
constructed from existing taxonomy structures, category hierarchies, or namespace
prefixes. Cross-domain relations like \texttt{analogous\_to} can be added
incrementally without disrupting existing data.

This migration path means CDC is an \emph{augmentation layer} that wraps
existing representations with modal world constraints, not a replacement that
requires discarding legacy knowledge. An RDF triple store can be converted to CDC
by mapping named graphs to domain specifications. An OWL ontology can be wrapped
by treating its namespace as a domain. A relational database can gain CDC
semantics by adding a single domain column to its fact tables.

CDC is therefore backward-compatible by construction: the worst case of ignoring
$@D$ is exactly the status quo.

\section{The Power of \textit{@D}: Domain as Modal World Constraint}
\label{sec:power}

This is the core of the paper. We do not merely define CDC's modal
interpretation---we demonstrate its power through concrete examples, then
formalise what makes that power possible.

\subsection{\textit{@D} Is a Constraint, Not a Label}

In standard Kripke semantics~\cite{kripke1963,chellas1980}, a modal frame is a
tuple $\langle W, R_{\mathit{acc}}\rangle$ where $W$ is a set of possible worlds
and $R_{\mathit{acc}}$ is an accessibility relation. The necessity operator
$\Box$ asserts that a proposition holds in all accessible worlds. CDC
operationalises this framework:

\begin{definition}[Modal Domain Interpretation]
\label{def:modal}
Each domain specification $d \in \mathcal{D}$ defines a possible world $w_d \in
W$. The CDC assertion
\[
  R@D(C,\,C') \;\equiv\; \Box_D\, R(C,\,C')
\]
states that relation $R$ between $C$ and $C'$ necessarily holds within the world
defined by domain $D$.
\end{definition}

This is not a notational convenience. It has teeth. Consider what happens when a
relation violates the laws of its declared world.

\medskip\noindent\textbf{Demonstration: Representation-level falsification.}
Suppose someone asserts:
\begin{lstlisting}[style=prologstyle]
causes(Thunder, Dark_Clouds, 'Meteorology')
\end{lstlisting}
In a traditional knowledge graph, this triple is syntactically valid and will be
stored without objection. No structural mechanism prevents the encoding of a
reversed causal claim.

In CDC, the \texttt{@Meteorology} annotation is not a passive tag---it is
$\Box_{\text{Meteorology}}$, a necessity operator that imports the constraints of
the meteorological world. Within $\Box_{\text{Meteorology}}$, the established
causal chain is:

\begin{center}\small
\begin{tabular}{c}
$\text{Moisture\_Accumulation}$\\[2pt]
{\color{gray}$\downarrow$}\enspace{\itshape causes@Met}\\[2pt]
$\text{Cloud\_Formation}$\\[2pt]
{\color{gray}$\downarrow$}\enspace{\itshape causes@Met}\\[2pt]
$\text{Charge\_Separation}$\\[2pt]
{\color{gray}$\downarrow$}\enspace{\itshape causes@Met}\\[2pt]
$\text{Lightning}$\\[2pt]
{\color{gray}$\downarrow$}\enspace{\itshape accompanies@Met}\\[2pt]
$\text{Thunder}$
\end{tabular}
\end{center}
The proposed assertion reverses this chain. Under $\Box_{\text{Meteorology}}$,
the causal direction from Thunder to Dark\_Clouds is inconsistent with the
world's constraint structure. The assertion is rejected---not by a separate rule
engine, not by a human reviewer, but by the modal semantics of the representation
itself.

This is the difference between a label and a constraint. A label says ``this
fact is about meteorology.'' A constraint says ``this fact must be consistent
with meteorology---and it isn't.''

No traditional knowledge graph---including Wikidata with qualifiers and RDF with
named graphs---can achieve this. They can annotate a triple with a domain tag,
but they cannot \emph{reject} a triple because its content violates the tagged
domain's semantic laws. CDC can, because $@D$ is not metadata. It is a modal
operator.

\subsection{A Priori Disambiguation: Context Resolves Before Reasoning Begins}

Traditional knowledge-graph pipelines handle polysemy through a multi-stage
process: entity linking, word sense disambiguation, context resolution---each a
separate module with its own error propagation. CDC eliminates this entire
pipeline.

\medskip\noindent\textbf{Demonstration: Disambiguation as a structural
property.} The concept ``Apple'' is ambiguous. In a traditional KG,
disambiguation requires either URI differentiation
(\texttt{Apple\_fruit} vs.\ \texttt{Apple\_company}) or post-hoc context
resolution. In CDC:
\[
  \Box_{\text{Medicine}}: \text{Apple} \to \text{Fruit} \to
  \text{Nutritional\_Component}
\]
\[
  \Box_{\text{Business}}: \text{Apple} \to \text{Company} \to
  \text{Technology\_Sector}
\]
The disambiguation is already complete at the moment of representation. When a
query enters the CDC knowledge base scoped to \texttt{@Medicine}, world
$w_{\text{Medicine}}$ is activated, and within that world, ``Apple'' necessarily
denotes a fruit.

Now consider a richer example. The same mathematical problem---computing the area
of a triangle with sides $a=5$, $b=6$, $c=7$---can be approached through
multiple methods. In CDC, each method lives in its own world:
\begin{align*}
  \text{Known\_Cond} &\xrightarrow{\;\text{apply@Geometry}\;[\text{Heron's Formula}]\;}
    \text{Semi-perimeter} \to \text{Area} \\
  \text{Known\_Cond} &\xrightarrow{\;\text{apply@Trigonometry}\;[\text{Law of Cosines}]\;}
    \text{Angle\_C} \to \text{Area} \\
  \text{Known\_Cond} &\xrightarrow{\;\text{apply@Analytic\_Geometry}\;[\text{Coords}]\;}
    \text{Coordinates} \to \text{Area} \\
  \text{Known\_Cond} &\xrightarrow{\;\text{apply@Linear\_Algebra}\;[\text{Cross Product}]\;}
    \text{Vector\_Product} \to \text{Area}
\end{align*}
Each path is locked to its domain world. The Geometry path uses only geometric
primitives; the Linear\_Algebra path uses only vector operations. There is no
risk of a hybrid step that mixes geometric theorems with algebraic identities in
an invalid way---each $\Box_D$ constrains the inference vocabulary to the tools
legitimate within that world. This is \emph{a priori} constraint: the
disambiguation and path-locking happen at representation time, not at query time.

\subsection{Cross-Domain Structural Transfer: When Worlds Are Compatible}

Perhaps the most striking power of CDC is its ability to support cross-domain
structural transfer---not as an ad hoc analogy, but as a formal consequence of
modal compatibility.

\medskip\noindent\textbf{Demonstration: Topology preservation across worlds.}
Consider a narrative about a grandmother raising a grandchild after family
tragedy. CDC extraction produces a structured semantic graph under domain
\texttt{@Intergenerational\_Caregiving}. Replacing the core domain with
\texttt{@Writing\_Pedagogy}---a teacher struggling to teach writing to
underprepared students---yields a graph with an identical six-node topology. The
two chains are shown side by side in Figure~\ref{fig:topology}.

\begin{figure}[ht]
\centering
\begin{minipage}[t]{0.46\textwidth}
  \centering
  \small{\textbf{Domain:} \texttt{@Intergenerational\_Caregiving}}
  \vspace{6pt}

  \chainnode{@Family\_Responsibility}{bearing burden, mutual dependence}
  \chainarrow{supports}
  \chainnode{@Nurturing\_Care}{daily care, illness care, dietary innovation}
  \chainarrow{manifests}
  \chainnode{@Character\_Education}{kindness, honesty, altruism}
  \chainarrow{reinforced by}
  \chainnode{@Academic\_Discipline}{strict standards, supervised practice}
  \chainarrow{produces}
  \chainnode{@Personal\_Qualities}{resilience, kindness, optimism}
  \chainarrow{generates}
  \chainnode{@Social\_Impact}{inspiring others}
\end{minipage}
\hfill
\begin{minipage}[t]{0.46\textwidth}
  \centering
  \small{\textbf{Domain:} \texttt{@Writing\_Pedagogy}}
  \vspace{6pt}

  \chainnode{@Professional\_Responsibility}{bearing pressure, persisting through difficulty}
  \chainarrow{supports}
  \chainnode{@Instructional\_Investment}{grading, tutoring, resource building}
  \chainarrow{manifests}
  \chainnode{@Character\_Education}{integrity, empathy, encouragement}
  \chainarrow{reinforced by}
  \chainnode{@Standards\_Enforcement}{strict formatting rules, supervised revision}
  \chainarrow{produces}
  \chainnode{@Professional\_Qualities}{resilience, patience, creativity}
  \chainarrow{generates}
  \chainnode{@Professional\_Resonance}{shared experience among colleagues}
\end{minipage}
\caption{Topology preservation across domain worlds. The six-node chain
  Responsibility $\to$ Care $\to$ Character $\to$ Discipline $\to$ Qualities
  $\to$ Impact is structurally invariant; only domain-specific content changes.}
\label{fig:topology}
\end{figure}

The topological structure is preserved exactly. The six-node chain
Responsibility $\to$ Care $\to$ Character $\to$ Discipline $\to$ Qualities
$\to$ Impact is invariant; only the domain-specific content at each node changes.
In modal terms, the transfer succeeds because:
\[
  \Diamond(\text{Intergenerational\_Caregiving}
           \wedge \text{Writing\_Pedagogy}):\;
  \text{structural\_alignment is satisfiable}
\]
The two worlds are modally compatible: their constraint structures do not
contradict when evaluated jointly. CDC formalises this via the
\texttt{analogous\_to} predicate:
\begin{lstlisting}[style=prologstyle]
analogous_to(Atom, Solar_System, 'Physics@Atomic', 'Astronomy@Planetary').
analogous_to(Neural_Network, Brain, 'CS@ML', 'Neuroscience@Cognition').
\end{lstlisting}

\subsection{Formal Properties}

\subsubsection{Modal Domain Separation}

\begin{theorem}[Modal Domain Separation]
\label{thm:separation}
Let $\mathcal{M} = \langle W, R_{\mathit{acc}}, V\rangle$ be a Kripke model
where each domain $d \in \mathcal{D}$ corresponds to a distinct world $w_d \in
W$. Then for any concept $c$, relation $r$, and target concepts $c'_1 \neq
c'_2$:
\[
  \Box_{D_1} r(c, c'_1) \;\wedge\; \Box_{D_2} r(c, c'_2)
  \;\wedge\; c'_1 \neq c'_2 \;\wedge\; D_1 \neq D_2
  \quad\text{is consistent.}
\]
\end{theorem}
\begin{proof}
Since $D_1 \neq D_2$, the assertions are evaluated in distinct worlds $w_{D_1}$
and $w_{D_2}$. The valuation function $V$ assigns truth values independently per
world: $V(w_{D_1}) \vDash r(c,c'_1)$ and $V(w_{D_2}) \vDash r(c,c'_2)$.
Because $w_{D_1} \neq w_{D_2}$, the two valuations do not interact, and no
contradiction arises even when $c'_1 \neq c'_2$. \qed
\end{proof}

This theorem is not a trivial observation that ``different namespaces don't
conflict.'' It is a formal guarantee rooted in modal semantics: the consistency
of divergent categorisations is a \emph{theorem} of the underlying logical
framework, not an ad hoc design choice. Example:

\begin{lstlisting}[style=prologstyle]
is_a(Apple, Fruit,    'Biology@Plant_Taxonomy').
is_a(Apple, Company,  'Business@Tech_Sector').
\end{lstlisting}

These coexist because $\Box_{\text{Biology}}$ and $\Box_{\text{Business}}$ define
independent worlds.

\subsubsection{Constraint Transferability}

\begin{property}[Constraint Transferability]
\label{prop:transfer}
Let $S$ be any computational system that consumes CDC representations as input.
If $S$ preserves the structural integrity of CDC triples (i.e., it reads $\langle
C, R@D, C'\rangle$ as a unit rather than discarding $D$), then $S$'s
computations are automatically scoped to the world $w_D$ defined by the domain
specification.
\end{property}

This property has a profound consequence: CDC representations are not passive
data awaiting interpretation by an external engine. They are \emph{active
constraint carriers} that shape the behaviour of any system that consumes them.
The difference from metadata-based approaches is fundamental. Metadata must be
explicitly queried and interpreted by each consuming system. Modal constraints
embedded in the representation itself require no such explicit
interpretation---they \emph{are} the representation.

\subsubsection{Cross-Domain Reasoning as Modal Possibility}

The cross-domain predicates receive precise modal semantics:

\medskip\noindent\textbf{Analogy} (structural correspondence between worlds):
\[
  \mathtt{analogous\_to}(C_1, C_2, D_1, D_2)
  \;\equiv\;
  \Diamond(D_1 \wedge D_2):\;\mathtt{structural\_alignment}(C_1, C_2)
\]

\noindent\textbf{Fusion} (constructing composite worlds):
\[
  \mathtt{fuses\_with}(C_1, C_2, C_{\mathit{new}}, D_1 \oplus D_2)
  \;\equiv\;
  \Box_{D_1 \oplus D_2}\;\mathtt{integrates}(C_1, C_2, C_{\mathit{new}})
\]
A new world is created in which the constraints of both parent worlds are
simultaneously satisfied.

\subsection{The Paradigm Shift: From Facts to Constrained Worlds}

Table~\ref{tab:paradigm} states precisely what CDC changes about knowledge
representation.

\begin{table}[ht]
\centering
\caption{Comparison of traditional KR and CDC along key representational
         dimensions.}
\label{tab:paradigm}
\small
\begin{tabularx}{\linewidth}{lXX}
\toprule
\textbf{Dimension} & \textbf{Traditional KR} & \textbf{CDC} \\
\midrule
What a triple asserts   & An unconditional fact
                        & A modally constrained assertion $\Box_D R(C,C')$ \\
Role of domain          & Absent or metadata
                        & First-class constraint operator \\
Disambiguation          & External multi-module pipeline
                        & Structural property of representation \\
Reasoning scope         & Global (all facts, all rules)
                        & World-local (per domain) \\
Falsification           & Not possible at representation layer
                        & Inconsistent assertions rejected by $\Box_D$ \\
Cross-domain reasoning  & Manual ontology alignment
                        & Modal possibility $\Diamond(D_1 \wedge D_2)$ \\
Constraint locus        & External rules and engines
                        & Embedded in representation itself \\
Consuming systems       & Must implement context resolution
                        & Inherit constraints automatically \\
\bottomrule
\end{tabularx}
\end{table}

This is a paradigm shift, not an incremental improvement. Traditional KR stores
facts; CDC stores constrained worlds. Traditional KR relies on external engines
to interpret context; CDC embeds context as structure. Traditional KR cannot
reject an ill-formed assertion at the representational layer; CDC can.

\section{Relation Predicate System}

CDC's expressive power arises not only from its modal domain structure but also
from a concise, semantically orthogonal relation predicate system.

\subsection{Design Philosophy}

Four principles guide the relation system:

\begin{enumerate}[leftmargin=*]
  \item \textbf{Compact core.} Approximately twenty predicates---compared to
  Wikidata's 9{,}000+ properties. Conceptual parsimony with sufficient
  expressiveness.
  \item \textbf{Semantic orthogonality.} Each relation encodes a distinct
  semantic function. No redundancy, no overlap.
  \item \textbf{Formal specification.} Every predicate has defined arity,
  algebraic properties (transitivity, symmetry, reflexivity), and computational
  complexity.
  \item \textbf{Cross-domain support.} Unlike traditional systems restricted to
  intra-domain reasoning, CDC includes first-class predicates that operate across
  domain worlds.
\end{enumerate}

\subsection{Structural Relations}

\textbf{\texttt{is\_a@D}} --- Taxonomic classification. Transitive, asymmetric.
All inferences are scoped by $\Box_D$:

\begin{lstlisting}[style=prologstyle]
% Transitive closure within a world
is_a_transitive(X, Z, D) :- is_a(X, Y, D), is_a(Y, Z, D).

% Attribute inheritance within a world
has_attribute(X, Attr, D) :- is_a(X, Y, D), has_attribute(Y, Attr, D).
\end{lstlisting}

The domain parameter $D$ ensures that inheritance chains never leak across
worlds: an attribute inherited under $\Box_{\text{Biology}}$ cannot contaminate
reasoning under $\Box_{\text{Business}}$.

\medskip\noindent\textbf{\texttt{part\_of@D}} --- Mereological relation.
Transitive, asymmetric. Represents part-whole hierarchies scoped to specific
worlds.

\medskip\noindent\textbf{\texttt{has\_attribute@D}} --- Property association.
Neither transitive nor symmetric.

\subsection{Logical Relations}

\textbf{\texttt{requires@D}} --- Prerequisite relation. Transitive, asymmetric,
acyclic. Within world $w_D$, prerequisite chains are computed recursively:

\begin{lstlisting}[style=prologstyle]
prerequisite_chain(X, Z, D) :-
    requires(X, Y, D), prerequisite_chain(Y, Z, D).
prerequisite_chain(X, Y, D) :- requires(X, Y, D).
\end{lstlisting}

\noindent\textbf{\texttt{cause\_of@D} and \texttt{enables@D}} --- Causal and
facilitative dependencies. The causal falsification power demonstrated in
Section~\ref{sec:power} derives directly from the modal scoping of
\texttt{cause\_of}: a causal claim must be consistent with the laws of world
$w_D$.

\medskip\noindent\textbf{\texttt{contrasts\_with@D}} --- Oppositional relation.
Symmetric, non-transitive.

\subsection{Cross-Domain Relations}

\textbf{\texttt{analogous\_to@}$D_1 \leftrightarrow D_2$} --- Structural analogy
between worlds, formalised as $\Diamond(D_1 \wedge D_2)$. Symmetric.

\begin{lstlisting}[style=prologstyle]
analogous_to(Atom, Solar_System, 'Physics@Atomic', 'Astronomy@Planetary').
analogous_to(Neural_Network, Brain, 'CS@ML', 'Neuroscience@Cognition').
\end{lstlisting}

\noindent\textbf{\texttt{fuses\_with@}$D_1 \oplus D_2$} --- Conceptual
integration creating composite worlds.

\begin{lstlisting}[style=prologstyle]
fuses_with(User_Experience, Technical_Feasibility,
           Product_Design, 'UX+Engineering').
\end{lstlisting}

\subsection{Properties Summary}

Table~\ref{tab:relations} summarises the algebraic properties and complexity of
each relation.

\begin{table}[ht]
\centering
\caption{Algebraic properties and complexity of CDC relations.}
\label{tab:relations}
\small
\begin{tabular}{lcccc}
\toprule
\textbf{Relation} & \textbf{Transitive} & \textbf{Symmetric}
                  & \textbf{Reflexive} & \textbf{Complexity} \\
\midrule
\texttt{is\_a}          & \checkmark & $\times$ & $\times$ & $O(n^2)$ \\
\texttt{part\_of}       & \checkmark & $\times$ & $\times$ & $O(n^2)$ \\
\texttt{requires}       & \checkmark & $\times$ & $\times$ & $O(n^2)$ \\
\texttt{analogous\_to}  & $\times$   & \checkmark & $\times$ & $O(n^2)$ \\
\texttt{fuses\_with}    & $\times$   & \checkmark & $\times$ & $O(n^3)$ \\
\bottomrule
\end{tabular}
\end{table}

\section{Reference Implementation}

CDC is not a theoretical proposal awaiting implementation. The modal constraints
formalised in Section~\ref{sec:power} are fully executable: the Prolog reference
implementation in this section translates every $\Box_D$ assertion into a
computable predicate, and the case studies in Section~\ref{sec:cases} demonstrate
real inference over real data structures. The gap between the formal framework
and its computational realisation is zero---every property claimed in this paper
can be verified by running the provided code.

\subsection{Prolog as Modal Constraint Checker}

While CDC is substrate-agnostic, Prolog serves as a natural reference
implementation~\cite{kowalski1974}: each Prolog predicate directly realises a
world-scoped modal assertion. The rule:
\begin{lstlisting}[style=prologstyle]
is_a(X, Y, Domain) :- ...
\end{lstlisting}
corresponds to the modal axiom: ``within world \texttt{Domain}, if conditions
$\varphi$ hold, then $\mathtt{is\_a}(X,Y)$ necessarily holds.'' Prolog's
resolution mechanism automates derivation within each world. The domain parameter
in every predicate is not an optional annotation---it is the computational
realisation of $\Box_D$.

However, CDC is not inherently tied to Prolog. The same model can be instantiated
in RDF triple stores with domain-qualified predicates, in SQL with domain tables
and joins, or in property graph databases supporting multi-context edges.
Section~\ref{sec:comparison} demonstrates this substrate independence through
side-by-side comparison.

\subsection{Core Predicates}

\begin{lstlisting}[style=prologstyle]
% Structural relations (world-scoped)
:- dynamic is_a/3.
:- dynamic part_of/3.
:- dynamic has_attribute/3.

% Logical relations (world-scoped)
:- dynamic requires/3.
:- dynamic cause_of/3.
:- dynamic enables/3.

% Cross-domain relations (multi-world)
:- dynamic analogous_to/4.
:- dynamic fuses_with/4.

% Context relations
:- dynamic context_value/3.
:- dynamic strategy/3.
\end{lstlisting}

\subsection{Inference Rules}

All inference is automatically scoped to world boundaries:

\begin{lstlisting}[style=prologstyle]
% Transitive closure (stays within world w_D)
is_a_star(X, Y, Domain) :- is_a(X, Y, Domain).
is_a_star(X, Z, Domain) :-
    is_a(X, Y, Domain),
    is_a_star(Y, Z, Domain).

% Prerequisite chains (stays within world w_D)
requires_star(X, Y, Domain) :- requires(X, Y, Domain).
requires_star(X, Z, Domain) :-
    requires(X, Y, Domain),
    requires_star(Y, Z, Domain).

all_prerequisites(Target, Domain, Prereqs) :-
    findall(P, requires_star(Target, P, Domain), Prereqs).
\end{lstlisting}

\subsection{Query Examples}

\begin{lstlisting}[style=prologstyle]
% Taxonomic ancestors within world w_math@algebra
?- is_a_star(quadratic_function, Supertype, 'math@algebra').

% Prerequisite chain within world w_highschool
?- all_prerequisites(calculus, 'highschool', Prereqs).

% Cross-domain analogy (modal possibility check)
?- analogous_to(neural_network, BioConcept, 'ai@ml', BioDomain).
\end{lstlisting}

\section{Representational Comparison: CDC vs.\ Traditional Approaches}
\label{sec:comparison}

Before presenting case studies, we demonstrate concretely what CDC's $@D$
achieves that traditional approaches cannot---not through theoretical argument,
but through side-by-side query comparison. The key claim is: CDC's domain
constraint provides structural pruning at the representation layer, with
negligible storage overhead.

\subsection{Storage Overhead: One String Field}

At the storage level, CDC adds exactly one field to each triple: a domain
specification string. Table~\ref{tab:storage} compares the per-triple storage
cost across representations.

\begin{table}[ht]
\centering
\caption{Per-triple storage comparison across representations.}
\label{tab:storage}
\small
\begin{tabular}{llcc}
\toprule
\textbf{Representation} & \textbf{Triple Structure} & \textbf{Fields}
                         & \textbf{Overhead vs.\ RDF} \\
\midrule
RDF Triple         & $\langle s, p, o \rangle$                & 3            & baseline \\
RDF Named Graph    & $\langle s, p, o, g \rangle$             & 4            & $+1$ graph URI \\
Wikidata Qualifier & $\langle s, p, o \rangle$ + qualifier list & $3+\text{var}$ & $+N$ qualifier pairs \\
CDC                & $\langle c, r, c', d \rangle$            & 4            & $+1$ domain string \\
\bottomrule
\end{tabular}
\end{table}

CDC's overhead is identical to RDF Named Graphs: one additional field per triple.
It is less than Wikidata qualifiers, which attach variable-length key-value pairs.
The representational cost is trivial. The question is what this one field buys.

\subsection{Query Comparison: The Same Question, Three Representations}

Consider a knowledge base containing information about ``Apple'' in both
biological and business contexts. A user queries: ``What is Apple?''

\subsubsection{In RDF (No Domain)}

\begin{lstlisting}[style=sparqlstyle]
# RDF: no domain scoping -- returns everything
SELECT ?type WHERE {
  :Apple rdf:type ?type .
}
\end{lstlisting}

\noindent\textit{Result:} \{\texttt{Fruit}, \texttt{Company}\}---mixed, unscoped.
The consuming application must implement its own disambiguation logic. To scope
the query, the application must add post-hoc filtering:

\begin{lstlisting}[style=sparqlstyle]
# RDF: manual filtering by ontology graph
SELECT ?type WHERE {
  GRAPH <http://bio.example.org> {
    :Apple rdf:type ?type .
  }
}
\end{lstlisting}

This works---but the scoping is in the \emph{query}, not in the
\emph{representation}. Every application that queries this data must independently
implement domain filtering.

\subsubsection{In Wikidata (Qualifiers)}

\begin{lstlisting}[style=sparqlstyle]
# Wikidata: qualifier-annotated statement
# Apple -- instance of -- Fruit  [context: Biology][subfield: Plant Taxonomy]
# Apple -- instance of -- Company [context: Business][subfield: Technology Industry]

SELECT ?type WHERE {
  wd:Apple p:P31 ?statement .
  ?statement ps:P31 ?type .
  ?statement pq:P642 wd:Q420 .  # "of" qualifier = Biology
}
\end{lstlisting}

This is more expressive than bare RDF. But qualifiers are annotations
\emph{about} statements, not structural constraints. The critical differences:

\begin{itemize}
  \item \textbf{Qualifiers cannot reject an inconsistent assertion.} If someone
  adds \texttt{Apple -- causes -- Cancer [context: Biology]}, the qualifier
  \texttt{[context: Biology]} does not and cannot evaluate whether this causal
  claim is valid within biology. It simply tags it.

  \item \textbf{Qualifiers do not participate in inference.} Transitive closure,
  attribute inheritance, and prerequisite chains do not automatically respect
  qualifier values---each inference rule must be manually extended to check
  qualifiers.

  \item \textbf{Qualifiers are invisible to basic queries.} A query
  \texttt{SELECT ?type WHERE \{ wd:Apple wdt:P31 ?type \}} (using the ``truthy''
  shorthand) ignores qualifiers entirely, returning the unscoped result.
\end{itemize}

\subsubsection{In CDC}

\begin{lstlisting}[style=prologstyle]
% CDC: domain is structural
is_a(apple, fruit,   'Biology@Plant_Taxonomy').
is_a(apple, company, 'Business@Technology_Industry').

% Query:
?- is_a(apple, What, 'Biology@Plant_Taxonomy').
What = fruit.
\end{lstlisting}

One query. One result. No post-hoc filtering. No qualifier parsing. No
application-level disambiguation. The domain constraint is in the data, so the
query result is pre-scoped.

Now the critical test---transitive inference:

\begin{lstlisting}[style=prologstyle]
% Knowledge base
is_a(apple,        fruit,        'Biology@Plant_Taxonomy').
is_a(fruit,        plant_product,'Biology@Plant_Taxonomy').
is_a(plant_product,organic_matter,'Biology@Plant_Taxonomy').
is_a(apple,        company,      'Business@Technology_Industry').
is_a(company,      corporation,  'Business@Technology_Industry').

% Query: all ancestors in the biology world
?- is_a_star(apple, Ancestor, 'Biology@Plant_Taxonomy').
Ancestor = fruit ;
Ancestor = plant_product ;
Ancestor = organic_matter.
\end{lstlisting}

The transitive closure automatically stays within the biology world. It never
returns \texttt{corporation}---not because a filter removed it, but because
\texttt{corporation} does not exist in world $w_{\text{Biology}}$.

\subsection{Query Space Pruning}

The structural pruning effect is a direct computational consequence. Consider a
knowledge base with $N$ total triples distributed across $K$ domains, with
approximately $N/K$ triples per domain.

\begin{table}[ht]
\centering
\caption{Query complexity comparison between RDF and CDC.}
\label{tab:complexity}
\small
\begin{tabular}{lll}
\toprule
\textbf{Operation} & \textbf{RDF (no domain)} & \textbf{CDC} \\
\midrule
Simple lookup      & $O(N)$ scan or $O(1)$ indexed
                   & $O(N/K)$ scan or $O(1)$ indexed \\
Transitive closure & $O(N^2)$ over all triples
                   & $O((N/K)^2)$ within one world \\
Cross-domain search & N/A (no formal mechanism)
                   & $O(N/K_1 \times N/K_2)$ between two worlds \\
\bottomrule
\end{tabular}
\end{table}

For a knowledge base with 100\,K triples across 50 domains:
\begin{itemize}
  \item RDF transitive closure searches over 100\,K triples: $O(10^{10})$.
  \item CDC transitive closure within one domain searches over $\sim$2\,K
  triples: $O(4 \times 10^6)$.
  \item Estimated reduction: $\sim$2{,}500$\times$ for transitive closure
  operations.
\end{itemize}

This is not a new indexing technique---it is a structural property of the
representation. The domain field partitions the search space at the data level,
not at the query level.

\subsection{What Qualifiers Cannot Do: A Summary}

Table~\ref{tab:capability} summarises the capability gap between RDF, Wikidata
qualifiers, and CDC's $@D$.

\begin{table}[ht]
\centering
\caption{Capability comparison: RDF, Wikidata qualifiers, and CDC.}
\label{tab:capability}
\small
\begin{tabularx}{\linewidth}{p{0.26\linewidth}XXX}
\toprule
\textbf{Capability} & \textbf{RDF} & \textbf{Wikidata} & \textbf{CDC} \\
\midrule
Scope relation to domain
  & Via separate graphs (manual)
  & Via annotation (not structural)
  & Structural (part of tuple) \\
Reject inconsistent assertions
  & No & No & Yes ($\Box_D$ constraint) \\
Transitive closure respects domain
  & Only if separate graphs
  & Requires custom extensions
  & Automatic \\
A priori disambiguation
  & No & No & Yes (pre-scoped at repr.\ time) \\
Cross-domain analogy
  & No formal mechanism
  & No formal mechanism
  & \texttt{analogous\_to@}$D_1{\leftrightarrow}D_2$ \\
Storage overhead per triple
  & baseline & $+N$ qualifier pairs & $+1$ string field \\
Query complexity (domain-scoped)
  & $=$ unscoped $+$ filter
  & $=$ unscoped $+$ qualifier check
  & Reduced by factor $K$ \\
\bottomrule
\end{tabularx}
\end{table}

\subsection{Interoperability: Bidirectional CDC $\leftrightarrow$ RDF Mapping}

CDC is designed to interoperate with existing Semantic Web infrastructure, not
replace it. The following bidirectional mapping rules define a lossless round-trip
between CDC four-tuples and RDF Named Graph quads.

\medskip\noindent\textbf{CDC $\to$ RDF Named Graph:}
\begin{lstlisting}[style=sparqlstyle]
CDC:  r(C, C', D)
RDF:  GRAPH <urn:cdc:domain:D> { C r C' }
\end{lstlisting}

\noindent\textbf{RDF Named Graph $\to$ CDC:}
\begin{lstlisting}[style=sparqlstyle]
RDF:  GRAPH <g> { S P O }
CDC:  P(S, O, extract_domain(g))
\end{lstlisting}

The graph URI is parsed into a domain specification. If the graph URI encodes
hierarchical structure (e.g., \texttt{<urn:bio:plant\_taxonomy>}), it maps to a
hierarchical domain string (\texttt{'Biology@Plant\_Taxonomy'}).

\begin{table}[ht]
\centering
\caption{Properties preserved by the RDF Named Graph--CDC mapping.}
\label{tab:mapping}
\small
\begin{tabular}{lcc}
\toprule
\textbf{Property}       & \textbf{RDF Named Graph} & \textbf{CDC (via mapping)} \\
\midrule
Statement grouping      & \checkmark & \checkmark \\
SPARQL queryability     & \checkmark & \checkmark \\
Inference scoping       & $\times$ (manual \texttt{GRAPH} clauses) & \checkmark (automatic) \\
Falsification           & $\times$ & \checkmark ($\Box_D$ semantics) \\
Cross-graph reasoning   & $\times$ & \checkmark (\texttt{analogous\_to}, \texttt{fuses\_with}) \\
\bottomrule
\end{tabular}
\end{table}

\section{Case Studies}
\label{sec:cases}

We present three proof-of-concept applications. Each demonstrates a specific
power of $@D$ that cannot be replicated by traditional approaches.

\subsection{Education: Domain-Scoped Personalisation}

\textbf{Scenario.} An online platform teaches programming to students with
diverse backgrounds.

\begin{lstlisting}[style=prologstyle]
% Concept in world w_cs@fundamentals
is_a(function, programming_concept, 'cs@fundamentals').

% Learner profiles define worlds
context_value(student_alice, math_background,   'student@profile').
context_value(student_bob,   design_background, 'student@profile').

% Strategies exist only in their respective worlds
strategy(explain_function, use_formal_definition,  'math_background@cs').
strategy(explain_function, use_workflow_metaphor,   'design_background@cs').

% Cross-domain analogy
analogous_to(function, machine, 'cs@programming', 'engineering@systems').
\end{lstlisting}

\textbf{What $@D$ buys.} Each learner's background defines a world. A query
\texttt{strategy(\allowbreak explain\_function,\allowbreak{} How,\allowbreak{} 'design\_background@cs')}
returns only \texttt{use\_workflow\_metaphor} without ever touching the
math-background strategies.

\begin{lstlisting}[style=prologstyle]
?- strategy(explain_function, Method, 'design_background@cs').
Method = use_workflow_metaphor.

?- analogous_to(function, What, 'cs@programming', 'engineering@systems').
What = machine.
\end{lstlisting}

\subsection{Enterprise: Cross-Department Integration}

\textbf{Scenario.} Product, Engineering, and Design teams use different
vocabularies. ``User story'' (Product) and ``functional requirement''
(Engineering) refer to structurally similar concepts in different department
worlds.

\begin{lstlisting}[style=prologstyle]
% Cross-department analogy
analogous_to(user_story, functional_requirement,
             'product@requirements', 'engineering@specs').

% Composite world: both teams' constraints apply
fuses_with(user_experience, technical_feasibility,
           integrated_product_spec, 'product+engineering').

% Conflict detection within composite world
conflicts_with(real_time_sync, battery_efficiency,
               'product+engineering@mobile').
\end{lstlisting}

\textbf{What $@D$ buys.} The \texttt{analogous\_to} predicate asserts that the
two department worlds are structurally compatible ($\Diamond(\text{product}
\wedge \text{engineering})$ is satisfiable). The \texttt{fuses\_with} predicate
constructs a composite world $\Box_{\text{product}+\text{engineering}}$ where
both teams' constraints are simultaneously active---and within that world,
\texttt{conflicts\_with} can detect contradictions. A traditional KG with two
department ontologies would require a hand-built bridge ontology to achieve the
same; CDC achieves it through world composition.

\subsection{Technical Documentation: Version Evolution}

\textbf{Scenario.} React framework evolves across versions; documentation must
be version-specific without losing cross-version continuity.

\begin{lstlisting}[style=prologstyle]
% Version-specific knowledge in separate worlds
is_a(class_component,      component_type, 'react@pre16.8').
is_a(functional_component, component_type, 'react@16.8+@hooks').

% Evolution across world boundaries
evolves_to(class_component, functional_component, 'react@paradigm_shift').

% Cross-version analogy
analogous_to(component_did_mount, use_effect,
             'react@pre16.8', 'react@16.8+@hooks').

% Version-specific recommendation
if_then(mobile_app, use_lazy_loading, 'react@mobile@perf').
\end{lstlisting}

\textbf{What $@D$ buys.} A query about \texttt{component\_did\_mount} scoped to
\texttt{'react@16.8+@hooks'} returns nothing---because that concept does not
exist in the hooks world. Instead, the \texttt{analogous\_to} predicate directs
the user to \texttt{use\_effect}.

\begin{lstlisting}[style=prologstyle]
?- analogous_to(component_did_mount, HooksEquivalent,
                'react@pre16.8', 'react@16.8+@hooks').
HooksEquivalent = use_effect.

?- is_a(What, component_type, 'react@16.8+@hooks').
What = functional_component.
% Note: class_component is NOT returned -- it exists in a different world.
\end{lstlisting}

\section{Analysis and Discussion}

\subsection{Changes Made by CDC}

The demonstrations in Sections~\ref{sec:power} and~\ref{sec:comparison} reveal a
systematic capability gap. This gap is not a matter of degree---it is
architectural.

\begin{description}[leftmargin=*,style=nextline]
  \item[Falsification at the representation layer.]
  No traditional KG can reject \texttt{causes(Thunder, Dark\_Clouds)} based on
  its content. CDC rejects it because $\Box_{\text{Meteorology}}$ imposes causal
  direction constraints. Knowledge quality becomes a structural property, not a
  procedural one.

  \item[Zero-cost disambiguation.]
  Traditional pipelines require entity linking $\to$ sense disambiguation $\to$
  context resolution (three modules, compounding error rates). CDC's $@D$
  pre-selects the semantic pathway at representation time. The disambiguation is
  complete before any query is issued.

  \item[Structural pruning without indexing tricks.]
  CDC's domain parameter partitions the search space at the data level.
  Transitive closure within one domain world operates on $N/K$ triples instead
  of $N$---a reduction that scales linearly with the number of domains.

  \item[Automatic inference scoping.]
  In CDC, the domain parameter in every inference rule ensures that chains never
  cross world boundaries unless explicitly bridged. In RDF/SPARQL, inference
  scoping requires manual \texttt{GRAPH} clauses or \texttt{FILTER} statements at
  every step---and forgetting one step produces silent cross-contamination.
\end{description}

\subsection{Representational Overhead and Adoption Cost}

CDC adds one string field per triple. This is the same overhead as RDF Named
Graphs and less than Wikidata qualifiers. The storage cost is negligible; the
structural benefit is substantial.

More importantly, CDC does not require any machinery beyond what traditional KR
already uses. Prolog predicates gain a domain parameter; SPARQL queries gain a
domain filter; SQL tables gain a domain column. The implementation complexity is
minimal because the constraint is in the data, not in the engine.

The adoption model is low-invasive by design: existing data is preserved without
modification, existing queries remain valid (they simply operate in an unscoped
mode equivalent to the status quo), and the semantic upgrade is incremental. The
bidirectional mapping with RDF Named Graphs further ensures that CDC-enhanced
data can be exported back to standard formats for interoperability with non-CDC
systems.

\subsection{Constraint Transferability}

The $@D$ constraint is not confined to Prolog. Because domain is a structural
component of the tuple rather than externally attached metadata, any system that
reads CDC triples as four-field units inherits the constraint automatically. A
SQL query with \texttt{WHERE domain = 'Biology@Plant\_Taxonomy'} achieves the
same world-scoping as a Prolog query with the domain parameter. Likewise, a
SPARQL query over CDC data in an RDF quad store achieves it through the named
graph mechanism. The constraint travels with the data: wherever the data goes,
the scoping goes with it.

\subsection{Limitations and Open Challenges}

\begin{description}[leftmargin=*,style=nextline]
  \item[Domain specification ambiguity.]
  Domains are free-form strings without formal equivalence detection.
  \texttt{'Biology@\allowbreak{}Plant\_Taxonomy'} and \texttt{'Bio@Plants'} might refer to
  the same world but are treated as distinct. Mitigation strategies include
  meta-level domain ontologies, string similarity metrics, and community
  conventions. A formal domain algebra---with subsumption, composition, and
  equivalence operators---is a priority for future formalisation.

  \item[Scope of modal formalisation.]
  This paper uses standard Kripke semantics with $\Box$ and $\Diamond$. More
  expressive frameworks---including multi-modal logics, probabilistic modal
  extensions (e.g., $\mathtt{is\_a}(X, Y, D, 0.85)$ for confidence-weighted
  assertions), and temporal modal operators for knowledge evolution---remain open
  directions.

  \item[Scalability.]
  Case studies involve 100--1{,}000 triples. Web-scale validation, including
  distributed CDC deployment and real-time inference over millions of triples
  across thousands of domains, remains to be empirically demonstrated, though the
  theoretical pruning analysis suggests favourable scaling properties.

  \item[Relation completeness.]
  The current $20+$ relations cover common use cases. Probabilistic reasoning,
  quantification, and deontic modalities are outside the current scope but could
  be integrated as extended relation families.

  \item[Empirical evaluation.]
  The current work is primarily theoretical and proof-of-concept. Systematic
  empirical evaluation---including benchmark datasets, baseline comparisons on
  established KR tasks, and user studies---is necessary to validate the practical
  benefits suggested by the theoretical analysis.

  \item[Computational semantics.]
  This paper establishes the representational semantics of $@D$ as a modal
  constraint. A full computational treatment---including decidability analysis of
  domain-scoped inference and complexity bounds for cross-domain
  reasoning---constitutes a natural next step.
\end{description}

\subsection{Future Research Directions}

\textbf{Theoretical directions} include formalising a domain algebra with
subsumption, composition, and similarity operators; developing probabilistic CDC
variants; introducing temporal extensions for knowledge evolution tracking; and
investigating the formal relationship between CDC's modal semantics and
Description Logic's epistemic extensions.

\textbf{Computational directions} include distributed CDC storage and partitioning
strategies (where domain-based partitioning is a natural fit), graph-database
integration, visual editors and validation tooling, and RDF/OWL interoperability
layers.

CDC's four-tuple structure is directly serialisable as RDF Named Graph quads with
semantic enhancement, and its $\Box_D$ operator is formally compatible with
Description Logic's existing modal extensions. Together, these properties suggest
a path toward standardisation within existing W3C infrastructure---potentially as
a semantic layer specification that enriches named graphs with modal constraint
semantics, without requiring changes to the underlying RDF data model or SPARQL
query language.

Finally, an important bridging direction is investigating how CDC's constraint
transferability property interacts with large language models, where the $@D$
annotation may serve as a natural constraint interface for structured generation
and verification.

\section{Conclusion}

This paper proposed the Domain-Contextualized Concept Graph (CDC) and established
its semantic foundation through modal logic. The core insight is that the domain
annotation $@D$ is not metadata---it is a modal necessity operator $\Box_D$ that
determines the conditions under which relations hold.

This single design decision---adding one string field to each triple---produces
representational power that traditional knowledge graphs cannot achieve:

\begin{itemize}
  \item Assertions that violate their declared domain's constraints are rejected
  at the representational layer, not by external rule engines.
  \item Disambiguation is a zero-cost structural property: $@D$ pre-selects
  semantic pathways at representation time, eliminating the need for multi-module
  disambiguation pipelines.
  \item Transitive inference automatically respects domain boundaries, with a
  structural search space reduction proportional to the number of domains.
  \item Cross-domain reasoning is formalised as modal possibility ($\Diamond$),
  with explicit predicates for analogy and conceptual fusion---capabilities absent
  from RDF, Wikidata, and traditional ontology systems.
  \item The constraint travels with the data: any system that reads CDC
  four-tuples inherits the scoping, regardless of its internal architecture.
\end{itemize}

The cost is one additional field per triple---the same as RDF Named Graphs, less
than Wikidata qualifiers. The adoption path is non-disruptive: any existing
knowledge base can be augmented with domain constraints without rebuilding from
scratch. The benefit is a paradigm shift: knowledge representation moves from
passive fact recording to actively constrained semantic structure, where $@D$ is
the lock, not the name tag.


\appendix

\section{Migration Examples}
\label{app:migration}

This appendix demonstrates that CDC's migration path (Section~3.4) is not
hypothetical. We show concrete transformations from three real-world
representation formats to CDC. In each case, the transformation is mechanical:
existing entities become concepts, existing relations become predicates, and
existing organisational structures become domain specifications.

\subsection{ICD-11 Foundation $\to$ CDC}

The WHO's International Classification of Diseases, 11th Revision (ICD-11) is
built on a three-layer architecture: a Foundation semantic network (formalised in
OWL), linearisations for specific use cases, and a content model specifying
required attributes. The Foundation contains approximately 85\,000 entities
organised in polyhierarchies with multiple inheritance.

\medskip\noindent\textbf{Source: ICD-11 Foundation (OWL).}

\begin{lstlisting}[style=genericstyle,language=XML,caption={ICD-11 Foundation OWL (simplified).}]
<owl:Class rdf:about="http://id.who.int/icd/entity/CA40">
  <rdfs:label>Pneumonia</rdfs:label>
  <rdfs:subClassOf rdf:resource="...CA40-CA43"/>
</owl:Class>
<owl:Class rdf:about="http://id.who.int/icd/entity/CA40.0">
  <rdfs:label>Bacterial pneumonia</rdfs:label>
  <rdfs:subClassOf rdf:resource="...CA40"/>
</owl:Class>
<owl:Class rdf:about="http://id.who.int/icd/entity/CA40.00">
  <rdfs:label>Pneumonia due to Streptococcus pneumoniae</rdfs:label>
  <rdfs:subClassOf rdf:resource="...CA40.0"/>
  <icd:hasCausingCondition rdf:resource="...XN0CS"/>
</owl:Class>
\end{lstlisting}

\textbf{Problem with OWL-only representation.} ICD-11's Foundation uses multiple
inheritance: the same entity can appear under different parent hierarchies. OWL
cannot express ``this is a respiratory disease when viewed from the anatomical
perspective and a bacterial infection when viewed from the etiological
perspective.'' The classification context is implicit.

\medskip\noindent\textbf{CDC transformation:}

\begin{lstlisting}[style=prologstyle]
% Anatomical classification world
is_a('CA40.00', 'CA40.0',           'ICD11@Respiratory@Anatomical').
is_a('CA40.0',  'CA40',             'ICD11@Respiratory@Anatomical').
is_a('CA40',    'Respiratory_Diseases','ICD11@Respiratory@Anatomical').

% Etiological classification world
is_a('CA40.00', 'Bacterial_Pneumonia', 'ICD11@Infectious@Etiological').
is_a('Bacterial_Pneumonia','Bacterial_Infections','ICD11@Infectious@Etiological').

% Causal relation (world-scoped)
cause_of('Streptococcus_pneumoniae','CA40.00','ICD11@Microbiology@Pathogenesis').

% Cross-classification analogy
analogous_to('CA40.00','CA40.00',
             'ICD11@Respiratory@Anatomical',
             'ICD11@Infectious@Etiological').

% Clinical severity (world-scoped)
has_attribute('CA40.00', severe,   'ICD11@Clinical@Severity').
has_attribute('CA40.00', bilateral,'ICD11@Clinical@Laterality').
\end{lstlisting}

\textbf{What CDC adds that OWL cannot express:}

\begin{enumerate}
  \item \emph{Classification context is explicit.} OWL says ``CA40.00
  subClassOf CA40.0'' without specifying which classification perspective makes
  this true. CDC says ``CA40.00 is\_a CA40.0 in the anatomical world.''

  \item \emph{Multiple inheritance becomes multi-world membership.} Instead of
  one entity with multiple parents in a flat hierarchy, CDC represents the same
  entity existing in multiple domain worlds, each with its own complete and
  non-conflicting taxonomy. The \texttt{analogous\_to} predicate explicitly
  bridges these worlds.

  \item \emph{Post-coordination gains modal scoping.} ICD-11's post-coordination
  maps to CDC's attribute system with explicit domain scoping. Each attribute axis
  lives in its own world.

  \item \emph{Linearisation becomes world selection.} ICD-11's linearisation
  process---choosing which parent to use for a particular statistical
  classification---corresponds in CDC to selecting a domain world.
\end{enumerate}

\textbf{Transformation rule (mechanical):}

\begin{lstlisting}[style=genericstyle]
For each OWL axiom: C rdfs:subClassOf D
  -> Determine classification context from hierarchy path
  -> Generate: is_a(C, D, 'ICD11@<context>')

For each OWL axiom with multiple parents:
  C rdfs:subClassOf D1, C rdfs:subClassOf D2
  -> If D1 and D2 belong to different classification axes:
       is_a(C, D1, 'ICD11@<axis1>')
       is_a(C, D2, 'ICD11@<axis2>')
       analogous_to(C, C, 'ICD11@<axis1>', 'ICD11@<axis2>')

For each OWL restriction: C hasCausingCondition P
  -> cause_of(P, C, 'ICD11@<relevant_domain>')
\end{lstlisting}

\subsection{OWL Domain Ontology $\to$ CDC: Pizza Ontology Example}

The Manchester OWL Pizza Ontology is a widely used teaching example in ontology
engineering. It classifies pizzas by their toppings, bases, and regional origin.

\begin{lstlisting}[style=genericstyle,language=XML,caption={OWL Pizza Ontology (simplified).}]
<owl:Class rdf:about="#MargheritaPizza">
  <rdfs:subClassOf rdf:resource="#NamedPizza"/>
  <rdfs:subClassOf>
    <owl:Restriction>
      <owl:onProperty rdf:resource="#hasTopping"/>
      <owl:someValuesFrom rdf:resource="#MozzarellaTopping"/>
    </owl:Restriction>
  </rdfs:subClassOf>
  ...
</owl:Class>
\end{lstlisting}

\textbf{CDC transformation:}

\begin{lstlisting}[style=prologstyle]
% Taxonomic hierarchy
is_a(margherita_pizza, named_pizza, 'Food@Pizza@Italian').
is_a(americana_pizza,  named_pizza, 'Food@Pizza@American').

% Toppings as domain-scoped attributes
has_attribute(margherita_pizza, mozzarella_topping, 'Food@Pizza@Toppings').
has_attribute(margherita_pizza, tomato_topping,     'Food@Pizza@Toppings').
has_attribute(americana_pizza,  pepperoni_topping,  'Food@Pizza@Toppings').

% Origin
has_attribute(margherita_pizza, italy,         'Food@Pizza@Origin').
has_attribute(americana_pizza,  united_states, 'Food@Pizza@Origin').

% Cross-domain: nutritional perspective
has_attribute(margherita_pizza, vegetarian,     'Food@Nutrition@DietaryCategory').
has_attribute(americana_pizza,  non_vegetarian, 'Food@Nutrition@DietaryCategory').

% Cross-domain analogy: culinary vs. cultural perspective
analogous_to(margherita_pizza, neapolitan_tradition,
             'Food@Pizza@Italian', 'Culture@Italy@Cuisine').
\end{lstlisting}

\textbf{Transformation rule:}

\begin{lstlisting}[style=genericstyle]
rdfs:subClassOf (simple)
  -> is_a(C, D, '<namespace>@<branch>')

owl:Restriction + onProperty P + someValuesFrom V
  -> has_attribute(C, V, '<namespace>@<property_domain>')

owl:Restriction + onProperty P + hasValue V
  -> has_attribute(C, V, '<namespace>@<property_domain>')
\end{lstlisting}

\subsection{Relational Database $\to$ CDC}

\textbf{Source: SQL relational schema.}

\begin{lstlisting}[style=genericstyle,language=SQL]
CREATE TABLE courses (
  course_id   VARCHAR(10) PRIMARY KEY,
  course_name VARCHAR(100),
  department  VARCHAR(50),
  level       VARCHAR(20)
);
CREATE TABLE prerequisites (
  course_id  VARCHAR(10),
  prereq_id  VARCHAR(10)
);

INSERT INTO courses VALUES ('CS101', 'Intro to Programming', 'CS', 'Undergrad');
INSERT INTO courses VALUES ('CS201', 'Data Structures',      'CS', 'Undergrad');
INSERT INTO courses VALUES ('MATH101','Calculus I',          'Math','Undergrad');
INSERT INTO prerequisites VALUES ('CS201', 'CS101');
INSERT INTO prerequisites VALUES ('CS201', 'MATH101');
\end{lstlisting}

\textbf{CDC transformation:}

\begin{lstlisting}[style=prologstyle]
is_a('CS101',  programming_course,    'CS@Undergraduate').
is_a('CS201',  data_structures_course,'CS@Undergraduate').
is_a('MATH101',calculus_course,       'Math@Undergraduate').

requires('CS201','CS101',   'CS@Undergraduate').
requires('CS201','MATH101', 'CS@Undergraduate@CrossDept').

% Cross-department analogy
analogous_to('MATH101','CS101',
             'Math@Undergraduate','CS@Undergraduate').

% Graduate-level recontextualisation
is_a('CS201', review_material, 'CS@Graduate').
\end{lstlisting}

\textbf{Transformation rule:}

\begin{lstlisting}[style=genericstyle]
For each row in a fact table:
  -> Entity columns     become concepts C, C'
  -> Relationship cols  become relation predicates r
  -> Category/dept cols become domain specifications D
  -> Generate: r(C, C', '<category>@<subcategory>')
\end{lstlisting}

\section{CDC Query Templates and Interaction Protocols}
\label{app:templates}

This appendix provides executable query templates that operationalise the modal
constraint mechanisms described in Section~\ref{sec:power}. Each template
corresponds to a specific theoretical property of CDC.

\subsection*{B.1 Single-Chain Constrained Reasoning}

\textit{Theoretical basis: A priori disambiguation (Section~4.2).} The $\Box_D$
operator locks each reasoning step to a single domain world.

\begin{lstlisting}[style=genericstyle,caption={Template B.1: Single-chain constrained reasoning.}]
You are a CDC reasoning assistant. Answer strictly in the following format:

Format:
ConceptA --{relation@domain [basis]}--> ConceptB
         --{relation@domain [basis]}--> ConceptC

Rules:
1. Every arrow must include an @domain annotation.
2. Every step must state its basis (theorem, reference, or principle).
3. If any step cannot be verified within the declared @domain, STOP
   and explain why.

Question: [user question]
\end{lstlisting}

\textbf{Example interaction:}

\begin{lstlisting}[style=genericstyle]
Question: A patient has fever 39C + severe headache + neck stiffness.
          Initial assessment?

Response:
Symptom_Cluster[fever + headache + neck_stiffness]
  --{classic_triad@Neurology [meningitis diagnostic criteria]}-->
  Suspected_Meningitis
  --{requires_workup@Emergency_Medicine [acute management protocol]}-->
  Lumbar_Puncture + Empiric_Antibiotics

Verification: All steps within declared @domains.
Confidence:   Steps grounded in established clinical criteria.
\end{lstlisting}

The \texttt{@domain} annotation at each step ensures that the reasoning chain
stays within a consistent clinical world. A step claiming
\texttt{--\{folk\_remedy@Neurology\}-->} would be rejected because folk remedies
are not licensed in $\Box_{\text{Neurology}}$.

\subsection*{B.2 Multi-Path Domain Comparison}

\textit{Theoretical basis: Domain Separation (Theorem~\ref{thm:separation}).}
The same question can be answered through multiple $\Box_D$ worlds, producing
non-contradictory but distinct reasoning paths.

\begin{lstlisting}[style=genericstyle,caption={Template B.2: Multi-path domain comparison.}]
You are a CDC reasoning assistant. Analyze all possible reasoning paths:

Question: [user question]

Requirements:
1. List at least 3 possible reasoning paths.
2. Each path uses CDC format with explicit @domain at every step.
3. Compare each path on: reliability, applicability, complexity.
4. Recommend the optimal path with justification.

Format:
Path 1: A --{R1@D1}--> B1 --{R2@D2}--> C1
Path 2: A --{R3@D3}--> B2 --{R4@D4}--> C2
...
Comparison:    [table]
Recommendation:[optimal path] because [reason]
\end{lstlisting}

This template operationalises Section~4.2's multi-path demonstration: the same
mathematical problem solved through $\Box_{\text{Geometry}}$,
$\Box_{\text{Trigonometry}}$, $\Box_{\text{Analytic\_Geometry}}$, and
$\Box_{\text{Linear\_Algebra}}$ produces four valid but distinct reasoning chains,
each locked to its domain world.

\subsection*{B.3 Causal Consistency Check}

\textit{Theoretical basis: Representation-level falsification (Section~4.1).}
The $\Box_D$ operator can reject assertions whose causal direction is inconsistent
with the declared domain's laws.

\begin{lstlisting}[style=genericstyle,caption={Template B.3: Causal consistency check.}]
You are a CDC reasoning assistant. Answer strictly in CDC format:

Format:
ConceptA --{relation@domain [basis]}--> ConceptB

Rules:
1. Every arrow must include @domain.
2. Every step must state its basis.
3. CRITICAL: If any step -- including concepts, relation, or domain --
   cannot be verified for consistency, accuracy, and causal direction
   within @domain, STOP immediately and explain why.

Question: [user question]
\end{lstlisting}

\textbf{Example interaction (falsification triggered):}

\begin{lstlisting}[style=genericstyle]
Question: Thunder causes dark clouds to appear.

Response:
I cannot construct this reasoning chain.
The proposed assertion:
  Thunder --{causes@Meteorology}--> Dark_Clouds
is inconsistent with Box_Meteorology.

Within the meteorological world, the established causal chain is:
  Moisture_Accumulation --{causes@Meteorology}--> Cloud_Formation
  Cloud_Formation       --{causes@Meteorology}--> Charge_Separation
  Charge_Separation     --{causes@Meteorology}--> Lightning
  Lightning             --{accompanies@Meteorology}--> Thunder

The proposed direction (Thunder -> Dark_Clouds) reverses the chain.
STOPPED: causal inconsistency within declared domain.
\end{lstlisting}

\subsection*{B.4 Notes on Template Usage}

These templates are not application-specific tools---they are executable
specifications of CDC's theoretical properties:

\begin{itemize}
  \item Template B.1 demonstrates that $\Box_D$ locks reasoning chains to a
  single world (a priori constraint).
  \item Template B.2 demonstrates that distinct $\Box_D$ worlds produce
  non-contradictory parallel paths (Domain Separation).
  \item Template B.3 demonstrates that $\Box_D$ can reject ill-formed assertions
  at the representation level (falsification).
\end{itemize}

Each template can be used with any general-purpose language model or rule-based
system. The constraint behaviour arises from the $@D$ structure embedded in the
query format, not from any special capability of the consuming system---a direct
manifestation of the constraint transferability property
(Property~\ref{prop:transfer}).


\newpage
\begin{thebibliography}{99}

\bibitem{arp2015}
Arp, R., Smith, B., \& Spear, A.~D. (2015).
\textit{Building Ontologies with Basic Formal Ontology}.
MIT Press.

\bibitem{auer2007}
Auer, S., Bizer, C., Kobilarov, G., et al. (2007).
DBpedia: A nucleus for a web of open data.
\textit{The Semantic Web}, 722--735.

\bibitem{baader2003}
Baader, F., Calvanese, D., McGuinness, D.~L., et al. (Eds.). (2003).
\textit{The Description Logic Handbook}.
Cambridge University Press.

\bibitem{baader2007}
Baader, F., et al. (2007).
\textit{The Description Logic Handbook} (2nd ed.).
Cambridge University Press.

\bibitem{barsalou1982}
Barsalou, L.~W. (1982).
Context-independent and context-dependent information in concepts.
\textit{Memory \& Cognition}, 10(1), 82--93.

\bibitem{bollacker2008}
Bollacker, K., Evans, C., Paritosh, P., et al. (2008).
Freebase: A collaboratively created graph database.
\textit{Proc.\ ACM SIGMOD}, 1247--1250.

\bibitem{bozzato2018}
Bozzato, L., Eiter, T., \& Serafini, L. (2018).
Enhancing context knowledge repositories with justifiable exceptions.
\textit{Artificial Intelligence}, 257.

\bibitem{carroll2005}
Carroll, J.~J., Bizer, C., Hayes, P., \& Stickler, P. (2005).
Named graphs, provenance and trust.
\textit{Proc.\ WWW}, 613--622.

\bibitem{chellas1980}
Chellas, B.~F. (1980).
\textit{Modal Logic: An Introduction}.
Cambridge University Press.

\bibitem{chute2022}
Chute, C.~G., et al. (2022).
Overview of ICD-11 architecture and structure.
\textit{BMC Medical Informatics and Decision Making}, 22(1), 104.

\bibitem{euzenat2013}
Euzenat, J., \& Shvaiko, P. (2013).
\textit{Ontology Matching} (2nd ed.).
Springer.

\bibitem{fagin2004}
Fagin, R., Halpern, J.~Y., Moses, Y., \& Vardi, M.~Y. (2004).
\textit{Reasoning about Knowledge}.
MIT Press.

\bibitem{fillmore1982}
Fillmore, C.~J. (1982).
Frame semantics.
\textit{Linguistics in the Morning Calm}, 111--137.

\bibitem{flouris2008}
Flouris, G., et al. (2008).
Ontology change: Classification and survey.
\textit{Knowledge Engineering Review}, 23(2), 117--147.

\bibitem{gardenfors2000}
G{\"a}rdenfors, P. (2000).
\textit{Conceptual Spaces: The Geometry of Thought}.
MIT Press.

\bibitem{goldberg2006}
Goldberg, A.~E. (2006).
\textit{Constructions at Work}.
Oxford University Press.

\bibitem{grau2008}
Grau, B.~C., et al. (2008).
Modular reuse of ontologies: Theory and practice.
\textit{JAIR}, 31, 273--318.

\bibitem{guha2016}
Guha, R.~V., Brickley, D., \& Macbeth, S. (2016).
Schema.org: Evolution of structured data on the web.
\textit{Commun.\ ACM}, 59(2), 44--51.

\bibitem{hogan2021}
Hogan, A., et al. (2021).
Knowledge graphs.
\textit{ACM Computing Surveys}, 54(4), 1--37.

\bibitem{horrocks2003}
Horrocks, I., Patel-Schneider, P.~F., \& Van Harmelen, F. (2003).
From SHIQ and RDF to OWL.
\textit{J.\ Web Semantics}, 1(1), 7--26.

\bibitem{kowalski1974}
Kowalski, R. (1974).
Predicate logic as programming language.
\textit{Proc.\ IFIP Congress}, 74, 569--574.

\bibitem{kripke1963}
Kripke, S.~A. (1963).
Semantical considerations on modal logic.
\textit{Acta Philosophica Fennica}, 16, 83--94.

\bibitem{lakoff1980}
Lakoff, G., \& Johnson, M. (1980).
\textit{Metaphors We Live By}.
University of Chicago Press.

\bibitem{lenat1995}
Lenat, D.~B. (1995).
CYC: A large-scale investment in knowledge infrastructure.
\textit{Commun.\ ACM}, 38(11), 33--38.

\bibitem{mccarthy1993}
McCarthy, J. (1993).
Notes on formalizing context.
\textit{Proc.\ IJCAI}, 555--560.

\bibitem{mccarthy1969}
McCarthy, J., \& Hayes, P.~J. (1969).
Some philosophical problems from the standpoint of AI\@.
\textit{Machine Intelligence}, 4, 463--502.

\bibitem{medin1988}
Medin, D.~L., \& Shoben, E.~J. (1988).
Context and structure in conceptual combination.
\textit{Cognitive Psychology}, 20(2), 158--190.

\bibitem{niles2001}
Niles, I., \& Pease, A. (2001).
Towards a standard upper ontology.
\textit{Proc.\ FOIS}, 2--9.

\bibitem{smith2007}
Smith, B., et al. (2007).
The OBO Foundry.
\textit{Nature Biotechnology}, 25(11), 1251--1255.

\bibitem{vrandecic2014}
Vrandečić, D., \& Krötzsch, M. (2014).
Wikidata: A free collaborative knowledgebase.
\textit{Commun.\ ACM}, 57(10), 78--85.

\bibitem{vogt2024}
Vogt, L., Kuhn, T., \& Hoehndorf, R. (2024).
Semantic units: Organising knowledge graphs into semantically meaningful units
of representation.
\textit{Journal of Biomedical Semantics}, 15, 7.

\end{thebibliography}
\end{document}